%% file: main.tex
\documentclass[twocolumn]{article}

\input{header.tex}

\title{Reinforcement Learning with Depreciating Assets}

\author{Taylor Dohmen}
\author{Ashutosh Trivedi}
\affil{University of Colorado Boulder}

\begin{document}
\maketitle

\begin{abstract}
	A basic assumption of traditional reinforcement learning is that the value of a reward does not change once it is received by an agent. 
  The present work forgoes this assumption and considers the situation where the value of a reward decays proportionally to the time elapsed since it was obtained. 
  Emphasizing the inflection point occurring at the time of payment, we use the term \emph{asset} to refer to a reward that is currently in the possession of an agent. 
  Adopting this language, we initiate the study of depreciating assets within the framework of infinite-horizon quantitative optimization. In particular, we propose a notion of asset depreciation, inspired by classical exponential discounting, where the value of an asset is scaled by a fixed discount factor at each time step after it is obtained by the agent. We formulate a Bellman-style equational characterization of optimality in this context and develop a model-free reinforcement learning approach to obtain optimal policies.
\end{abstract}

\section{Introduction}
\input{intro.tex}

\paragraph{Organization.}
We begin by introducing necessary notation and reviewing the relevant technical background.
Section~\ref{sec:deprec} develops results on discounted depreciating payoff, while Section~\ref{sec:average} develops results for the average depreciating objective. 
We discuss some closely related work in Section~\ref{sec:related} and recap our contributions in the concluding section. 
\section{Preliminaries}
\input{prelims.tex}

\section{Discounted Depreciating Payoff}
\label{sec:deprec}
\input{past-disc.tex}

\section{Related Work}
\label{sec:related}
Discounted and average payoffs have played central roles in the theory of optimal control and reinforcement learning.
A multitude of deep results exist connecting these objectives
\citep{BewleyKohlberg76,BewleyKohlberg78,MertensNeyman81,AnderssonMiltersen09,ChatterjeeDoyenSingh11,ChatterjeeMajumdar12,Ziliotto16,Ziliotto16G,Ziliotto18}
in addition to an extensive body of work on algorithms for related optimization problems and their complexity
\citep{FilarSchultz86,RaghavanFilar91,RaghavanSyed03,ChatterjeeMajumdarHenzinger08,ChatterjeeIbsenJensen15}.

The value for the depreciating assets is defined as a past discounted sum of rewards.
Past discounted sums for finite sequences were studied in the context of optimization \citep{alur2012regular} and are closely related to exponential recency weighted average, a technique used in nonstationary multi-armed bandit problems \citep{Sutton18} to estimate the average reward of different actions by giving more weight to recent outcomes.
However, to the best of our knowledge, depreciating assets have not been formally studied as a payoff function.

Discounted objectives have found significant applications in areas of program verification and synthesis
\citep{deAlfaroHenzingerMajumdar03,CernyChatterjeeHenzingerRadhakrishnaSing11}. 
Although the idea of past operators is quite old \citep{LichtensteinPnueliZuck85}, relatively recently a number of classical formalisms including temporal logics such as LTL and CTL and the modal $\mu$-calculus have been
extended with past-tense operators and with discounted quantitative semantics
\citep{deAlfaroFaellaHenzingerMajumdarStoelinga05,AlmagorBokerKupferman14,AlmagorBokerKupferman16,littma17}.
A particularly significant result \citep{Markey03} around LTL with classical
boolean semantics is that, while LTL with past operators is no more expressive
than standard LTL, it is exponentially more succinct. 
It remains open whether this type of relationship holds for other logics and their extensions by past operators when interpreted with discounted quantitative semantics \citep{AlmagorBokerKupferman16}. 

\section{Conclusion}
\label{sec:conclusion}
In the stochastic optimal control and reinforcement learning setting the agents select their actions to maximize a discounted payoff associated with the resulting sequence of scalar rewards. 
This interaction models the way dopamine driven organisms maximize their reward sequence based on their capability to delay gratification (discounting). While this paradigm provides a natural model in the context of streams of immediate rewards, when the valuations and objectives are defined in terms of assets that depreciate, the problem cannot be directly modeled in the classic framework.
We initiated the study of optimization and learning for the depreciating assets, and showed a surprising connection between these problems and traditional discounted problems.
Our result enables solving optimization problems under depreciation dynamics by tweaking the algorithmic infrastructure that has been extensively developed over the last several decades for classic optimization problems.

We believe that depreciating assets may provide a useful abstraction to a number of related problems.
The following points sketch some of these directions and state several problems that remain open.
\begin{itemize}[wide]
    \item[{$\blacktriangleright$}] Regret minimization \citep{Cesa-BianchiLugosi06} is a popular criterion in the setting of online
    learning where a decision-maker chooses her actions so as to minimize the average regret---the difference between the realized reward and the reward that could have been achieved.
    We posit that imperfect decision makers may view their regret in a depreciated sense, since a suboptimal action in the recent past tends to cause more regret than an equally suboptimal action in the distant past. 
    We hope that the results of this work spur further interest in developing foundations of past-discounted characterizations of regret in online learning and optimization.
    \item[{$\blacktriangleright$}] In solving multi-agent optimization problems, a practical assumption involves bounding the capability of any adversary by assuming that they have a limited memory of the history of interaction, and this can be modeled via a discounting of past outcomes.
    From our results it follows that two-player zero-sum games with depreciation dynamics under both discounted and average payoffs can be reduced to classic optimization games modulo some scaling of the immediate rewards.
    \item[{$\blacktriangleright$}] The notion of state-based discount factors has been studied in the context of classic optimization and learning.
    Is it possible to extend the results of this paper to the setting with state-dependent depreciation factors?
    This result does not directly follow from the tools developed in this paper, and it remains an open problem.
    \item[{$\blacktriangleright$}] Continuous-time MDPs provide a dense-time analog of discrete-time MDPs and optimization and RL algorithms for such systems are well understood.
    Is it possible to solve optimization and learning for CTMDPs with depreciating assets?
\end{itemize}



\bibliographystyle{plainnat}
\bibliography{references.bib}

\end{document}

%% file: header.tex
\usepackage{xcolor}
\colorlet{darkgreen}{green!60!black}
\usepackage{graphics}

\usepackage{fourier}
\usepackage[T1]{fontenc}

\usepackage{bm}
\usepackage{mathtools}
\usepackage{amsmath}
\usepackage{amssymb}

\usepackage[hidelinks]{hyperref}
\usepackage{cleveref}
\usepackage{authblk}
\usepackage[semicolon,square,sort&compress]{natbib}
\usepackage[margin=1.5cm,twocolumn]{geometry}
\usepackage{enumitem}

\usepackage{amsthm}
\usepackage{thmtools}
\declaretheoremstyle[
	headfont=\normalfont\scshape,
	bodyfont=\normalfont\itshape,
	postheadspace=1em,
	qed=\scalebox{1.5}{$\lrcorner$}
]{helper}
\declaretheorem[style=helper, numbered=no, name=\scalebox{1.5}{$\ulcorner$}\quad Mertens' Theorem]{mertens}
\declaretheorem{theorem}
\declaretheorem{example}
\declaretheorem{corollary}
\declaretheorem{lemma}

\usepackage{tikz}
\usetikzlibrary{
	automata,
	positioning,
	calc,
	shapes,
	arrows,
	arrows.meta,
	fit
}

\newcommand{\tn}[1]{\textnormal{#1}}
\newcommand{\bb}[1]{\mathbb{#1}}
\newcommand{\ms}[1]{\ensuremath{\mathsf{#1}}}
\newcommand{\mc}[1]{\mathcal{#1}}

\newcommand{\paren}[1]{\left( #1 \right)}
\newcommand{\card}[1]{\left| #1 \right|}
\newcommand{\seq}[1]{\left\langle #1 \right\rangle}

\newcommand{\bk}[1]{\left[ #1 \right]}

\DeclareMathOperator*{\argmax}{\arg\max}

\newcommand{\dist}[1]{\ms{Dist}\paren{#1}}

\newcommand{\Mid}{\:\middle\vert\:}

\newcommand{\cla}[1]{{\color{blue}#1}}
\newcommand{\clb}[1]{{\color{red}#1}}
\newcommand{\clc}[1]{{\color{teal}#1}}
\newcommand{\cld}[1]{{\color{violet}#1}}
\newcommand{\cle}[1]{{\color{purple}#1}}

%% file: intro.tex
\emph{Time preference}~\citep{lowen92,frederick2002time} refers to the tendency of rational agents to value potential \emph{desirable outcomes} in proportion to the expected time before such an outcome is realized.
In other words, agents prefer to get a future reward sooner rather than later, all else being equal, and similarly, agents prefer to experience negative outcomes later rather than sooner. 
This phenomenon is typically codified in mathematical models in terms of discounting~\citep{Shapley53} and has been applied to a diverse array of disciplines concerned with optimization such as economics~\citep{heal2007discounting,philibert1999economics}, game theory~\citep{FilarVrieze96}, control theory~\citep{Puterman94}, and reinforcement learning~\citep{Sutton18}.
These models focus on the situation in which an agent moves through a stochastic environment in discrete time by selecting an action to perform at each time step and receiving an immediate reward based on the selected action and environmental state.
In particular, we consider exponential discounting, as introduced by \citet{Shapley53}, in which the agent carries this process on ad infinitum to generate an infinite sequence of rewards $\seq{r_n}^\infty_{n=1}$ with the goal of maximizing, with respect to a discount factor $\lambda \in (0,1)$, the discounted sum $\sum^\infty_{n=1} \lambda^{n-1} r_n$.
The discount factor is selected as a parameter and quantifies the magnitude of the agent's time preference.

A notable characteristic of the aforementioned discounted optimization framework is an implicit assumption that the utility of a reward remains constant once it is obtained by a learning agent.
While this seemingly innocuous supposition simplifies the model and helps to make it amenable to analysis, 
there are a number of scenarios where such an assumption is not appropriate.
Consider, for instance, the most basic and ubiquitous of rewards used to incentivize human behaviors: money.
The value of money tends to decay with time according to the rate of inflation, and the consequences of this decay are a topic of wide spread interest and intense study~\citep{hulten1980measurement,comley2015inflation,beckerman1991economics,fergusson2010money}.
\emph{Recognizing the fundamental role such decay has in influencing the dynamics of economic systems throughout the world, we consider its implications with respect to optimization and reinforcement learning in Markov decision processes.}


\subsection{Asset Depreciation}
When discussing a situation with decaying reward values, it is useful to distinguish between potential future rewards and actual rewards that have been obtained.
As such, we introduce the term \emph{asset} to refer to a reward that has been obtained by an agent at a previous moment in time.
Using this terminology, the present work may be described as an inquiry into optimization and learning under the assumption that assets \emph{depreciate}.
Depreciation, a term borrowed from the field of finance and accounting \citep{Wright64,Burt72}, describes exactly the phenomenon where the value of something decays with time.

We propose a notion of depreciation that is inspired by traditional discounting and is based on applying the same basic principle of time preference to an agent's history in addition to its future.
More precisely, we consider the situation in which an agent's behavior is evaluated with respect to an infinite sequence of cumulative accrued assets, each of which is discounted in proportion to how long ago it was obtained.
That is, we propose evaluating the agent in terms of functions on the sequence of assets 
\[
\seq{\sum^n_{k=1} r_k \gamma^{n-k}}^\infty_{n=1},
\] 
where $\gamma \in (0,1)$ is a discount factor, rather than on the sequence of rewards $\seq{r_n}^\infty_{n=1}$.
To motivate the study of depreciation and illustrate its naturalness, we examine the following hypothetical case-study.

\begin{example}[Used Car Dealership]
	\label{ex:car}
Consider a used car dealership with a business model involving purchasing used cars in locations with favorable regional markets, driving them back to their shop, and selling them for profit in their local market.
	Suppose that our optimizing agent is an employee of this dealership, tasked with managing capital acquisition.
	More specifically, this employee's job is to decide the destination from which the next car should be purchased, whenever such a choice arises. 
	The objective of the agent is to maximize the sum of the values of all vehicles in stock at the dealership over a discounted time-horizon for some discount factor $\lambda \in (0,1)$.
	Note that the discounted time-horizon problem is equivalent to the problem of maximizing expected terminal payoff of the process given a constant probability $(1-\lambda)$ of terminating operations at any point.

	It has long been known~\citep{Wykof70,ackerman1973used} that cars tend to continually depreciate in value after being sold as new, and so any reasonable model for the value of all vehicles in the inventory should incorporate some notion of asset depreciation.
	Suppose that another discount factor $\gamma \in (0,1)$ captures the rate at which automobiles lose value per unit of time. 
  Considering $\gamma$-depreciated rewards and $\lambda$-discounted horizon, the goal of our agent can be defined as a \emph{discounted depreciating optimization problem}.
  Alternatively, one may seek to optimize the \emph{long run average} (mean payoff) of $\gamma$-depreciated rewards.
\end{example}
    
\subsection{Discounted Depreciating Payoff} 
	Consider the sequence $x = \seq{3, 4, 5, 3, 4, 5, \ldots}$ of (absolute) rewards accumulated by the agent.
	In the presence of depreciation, the cumulative asset values at various points in time follow the sequence
	\begin{gather*}
		3, (3\gamma + 4), (3 \gamma^2+4\gamma+5), (3 \gamma^3+4\gamma^2+5\gamma+3), \\
		(3 \gamma^4+4\gamma^3+5\gamma^2+3\gamma+4), \ldots
	\end{gather*}
	For the $\lambda$-discounted time horizon, the value of the assets can be computed as follows:
\begin{align*}
	&\cla{3} + \clb{\lambda (3\gamma  {+} 4}) + \clc{\lambda^2 (3 \gamma^2 {+}4\gamma {+}5}) + \cld{\lambda^3(3 \gamma^3 {+}4\gamma^2 {+}5\gamma {+}3)} +\\
	&\qquad  \cle{\lambda^4(3 \gamma^4 {+}4\gamma^3 {+}5\gamma^2 {+}3\gamma+4)} +  \ldots \\
	&= (\cla{3} {+} \clb{3 \lambda\gamma} {+}\clc{3\gamma^2\lambda^2} {+}\cdots) +(\clb{4 \lambda } {+}\clc{4 \lambda^2\gamma} {+}\cld{4\lambda^3\gamma^2} {+}\cdots) + \\
	&\qquad (\clc{5\lambda^2}  {+}\cld{5\lambda^3\gamma} {+} \cle{\lambda^5\gamma^2} {+} \cdots)+ (\cld{3\lambda^3}  {+} \cle{3 \lambda\gamma^4} {+}3\gamma^2\lambda^5 {+}\cdots)+\cdots \\
	&= 3(1  {+} \lambda\gamma {+}\gamma^2\lambda^2 {+}\cdots) +  4\lambda(1  {+} \lambda\gamma {+}\lambda^2\gamma^2 {+}\cdots) + \\
	&\qquad  5\lambda^2(1  {+}\lambda\gamma {+}\lambda^2\gamma^2 {+} \cdots)+ 3\lambda^3(1  {+} \lambda\gamma {+} \gamma^2\lambda^2 {+}\cdots) +\ldots\\
	&= \frac{3+4\lambda+5\lambda^2+3\lambda^3+\cdots}{(1-\lambda\gamma)} \\
	&= \frac{3+4\lambda+5\lambda^2}{(1-\lambda\gamma)(1-\lambda^3)}.
\end{align*}
Notice that this $\gamma$-depreciated sum is equal to the $\lambda$-discounted sum when immediate rewards are scaled by a factor $\frac{1}{1-\lambda\gamma}$.
We show that this is not a mere coincidence, and prove that this equality holds also for general MDPs.

\subsection{Average Depreciating Payoff} 
Next, consider the long-run average of the depreciating asset values as the limit inferior of the sequence 
\begin{gather*}
3, \frac{3\gamma {+} 4}{2}, \frac{3 \gamma^2 {+} 4\gamma {+} 5}{3},  \frac{3 \gamma^3 {+} 4\gamma^2 {+} 5\gamma {+} 3}{4}, \\
\frac{3 \gamma^4 {+} 4\gamma^3 {+} 5\gamma^2 {+} 3\gamma {+} 4}{5}, \ldots
\end{gather*}
Based on classical Tauberian results~\citep{BewleyKohlberg76}, it is tempting to conjecture that the $\lambda$-discounted, $\gamma$-depreciating value converges to this mean as $\lambda \to 1$, e.g.
\begin{align*}
	\lim_{\lambda\to 1} (1-\lambda) \frac{3+4\lambda+5\lambda^2}{(1-\lambda\gamma)(1-\lambda^3)} &= \lim_{\lambda\to	1} \frac{3+4\lambda+5\lambda^2}{(1-\lambda\gamma)(1+\lambda+\lambda^2)} \\
	&= \frac{3+4+5}{3(1-\gamma)}.
\end{align*}
Indeed, we prove that this conjecture holds.

\paragraph{Contributions.} The highlights of this paper are given below.
\begin{itemize}[wide]
    \item[{$\blacktriangleright$}] We initiate the study of discounted and average payoff optimization in the presence of depreciation dynamics.
    \item[{$\blacktriangleright$}] We characterize the optimal value of the discounted depreciating payoff via Bellman-style optimality equations and use them to show that stationary deterministic policies are sufficient for achieving optimality. 
    Moreover, our characterization enables computing the optimal value and an optimal policy in polynomial time in the planning setting. 
    \item [{$\blacktriangleright$}] The optimality equation also facilitates a formulation of a variant of Q-learning that is compatible with asset depreciation, thereby providing a model-free reinforcement learning approach to obtain optimal policies in the learning setting.
    \item [{$\blacktriangleright$}]
    We show the classical Tauberian theorem relating discounted and average objectives can be extended to the depreciating reward setting. This result allows us to establish the sufficiency of stationary deterministic policies for optimality with respect to the average depreciating payoffs.
\end{itemize}

%% file: prelims.tex
Let $\bb{R}$ be the set of real numbers and $\bb{N}$ the set of natural numbers. 
For a set $X$, we write $\card{X}$ to denote its cardinality and $\dist{X}$ for the set of all probability distributions over $X$.
A point distribution over $X$ is one that assigns probability 1 to a unique element of $X$ and probability 0 to all others.

The technical portions of the paper are carried out within the standard mathematical framework of asymptotic optimization and learning in environments modeled as finite Markov decision processes.
Our presentation follows the conventions set in the standard textbooks on the optimization and learning \citep{Puterman94,FilarVrieze96,SuttonBarto98,FeinbergShwartz12}.

\subsection{Markov Decision Processes}
A (finite) \emph{Markov decision process} (MDP) $M$ is a tuple $(S, A, T, R)$ in which $S$ is a finite set of states, $A$ is a finite set of actions, $T : \paren{S \times A} \to \dist{S}$ is a stochastic transition function specifying, for any $s,t \in S$ and $a \in A$ the conditional probability $T(t \mid s,a)$ of moving to state $t$ given that the current state is $s$ and that action $a$ has been chosen, and $R : \paren{S \times A} \to \bb{R}$ is a real-valued reward function mapping each state-action pair to a numerical valuation.
For any function $f : S \to \bb{R}$, i.e. any random variable on the state space of the MDP, we write $\bb{E}_T\bk{f(t) \mid s,a}$ to denote the conditional expectation $\sum_{t \in S} f(t) T(t \mid s,a)$ of $f$ on the successor state, given that the agent has selected action $a$ from state $s$.
A path in $M$ is a sequence $s_1 a_1 s_2 \cdots a_n s_{n+1}$ of alternating states and actions such that $0 < T(s_{k+1} \mid s_k, a_k)$ at every index.
Let $\mc{F}(M)$ denote the set of all finite paths in $M$ and $\mc{I}(M)$ denote the set of all infinite paths in $M$.

\paragraph{Payoffs, Policies, and Optimality.}
We focus on infinite duration quantitative optimization problems where an outcome may be concretized as an infinite path in the MDP.
Such an outcome is evaluated relative to some mapping into the real numbers $\mc{I}(M) \to \bb{R}$ called a payoff.
A policy on $M$ is a function $\pi : \mc{F}(M) \to \dist{A}$ that chooses an a distribution over the action set, given a finite path in $M$.
Fixing a policy $\pi$ induces, for each state $s$, a unique probability measure $\bb{P}^\pi_s$ on the probability space over the Borel subsets of $\mc{I}(M)$.
This enables the evaluation of a policy, modulo a payoff and initial state $s$, in expectation $\bb{E}^\pi_s$.
Let $\Pi^M$ be the set of all policies on the MDP $M$.
A policy is optimal for a payoff if it maximizes, amongst all other policies, the expected value of that payoff, and this maximal expectation is called the value of the payoff on $M$.

\paragraph{Strategic Complexity.}
The strategic complexity of a payoff characterizes the necessary structure required for a policy to be optimal.
A qualitative aspect of strategic complexity is based on whether or not there exist environments for which optimal policies are necessarily probabilistic (\emph{mixed}).
A policy is \emph{deterministic} (\emph{pure}) if returns a point distribution for every input.
A policy is stationary if $\pi(s_1 a_1 \cdots a_{n-1} s_n) = \pi(s_n)$ holds at every time $n$.
The class of deterministic stationary policies is of special interest since there are finitely many such policies on any finite MDP; we consider these policies as functions $S \to A$.

\subsection{Discounted and Average Payoffs}
Given a path $s_1 a_1 s_2 \cdots$ in an MDP, two well-studied objectives are the discounted payoff, relative to a discount factor $\lambda \in (0,1)$, and the average payoff, defined as
\begin{gather}
	\sum^\infty_{n=1} \lambda^{n-1} R(s_n, a_n), \text{ and } \tag{Discounted Payoff} \\
	\liminf_{n \to \infty} \frac{1}{n} \sum^n_{k=1} R(s_k, a_k). \tag{Average Payoff}
\end{gather}
The discounted value and average value functions are defined 
\begin{align}
	V_\lambda(s) &= \sup_{\pi \in \Pi^M} \bb{E}^\pi_s\bk{\sum^\infty_{n=1} \lambda^{n-1} R(s_n, a_n)}, \tag{Discounted Value} \\
	V(s) &= \sup_{\pi \in \Pi^M} \bb{E}^\pi_s\bk{\liminf_{n \to \infty} \sum^n_{k=1} \frac{R(s_k, a_k)}{n}}. \tag{Average Value}
\end{align}
A stronger notion of optimality, specific to the discounted payoff, is Blackwell optimality.
A policy $\pi$ is Blackwell optimal if there exists a discount factor $\lambda_0 \in (0,1)$ such that $\pi$ is optimal for the discounted payoff with any discount factor in the interval $[\lambda_0, 1)$.

An alternative characterization of the discounted value is as the unique solution to the optimality equation
\begin{equation*}
	V_\lambda(s) = \max_{a \in A} R(s, a) + \lambda\bb{E}_T\bk{V_\lambda(t) \mid s, a},
\end{equation*}
which is the starting point for establishing the following result on the complexity of discounted and average payoffs \citep{Puterman94,FeinbergShwartz12,FilarVrieze96}.

\begin{theorem}
    Both discounted and average payoffs permit deterministic stationary optimal policies. Moreover, optimal values for both payoffs can be computed in polynomial time. 
\end{theorem}

\subsection{Reinforcement Learning}
Reinforcement learning (RL)~\citep{Sutton18} is a sampling-based optimization paradigm based on the feedback received from the environment in the form of scalar rewards. 
The standard RL scenario assumes a discounted payoff, and model-free approaches typically leverage the \emph{state-action value} or \emph{Q-value}: defined as the optimal value from state $s$, given that action $a$ has been selected, and is the solution of the equation
\begin{equation*}
	Q_\lambda(s, a) = R(s, a) + \lambda \bb{E}_T\bk{ V_\lambda(t) \mid s, a}.
\end{equation*}
The Q-value provides the foundation for the classic Q-Learning algorithm \citep{WatkinsDayan92}, which learns an optimal policy by approximating $Q_\lambda$ with a sequence $Q^n_\lambda$ of maps which asymptotically converge to $Q_\lambda$.
In particular, $Q^1_\lambda$ is initialized arbitrarily and then the agent explores the environment by selecting action $a = \argmax_{a \in A} Q^n_\lambda(s, a)$ from the current state $s$ and performing the update
\begin{equation}
	\label{eq:Q-update}
	Q^{n+1}_\lambda(s, a) \gets Q^n_\lambda(s, a) + \alpha_n \paren{R(s, a) + \lambda V^n_\lambda(t) - Q^n_\lambda(s, a)},
\end{equation}
in which $t$ is the next state as determined by the outcome of sampling the conditional distribution $T(\cdot \mid s, a)$, the family of $\alpha_n \in (0,1)$ are time-dependent parameters called learning rates, and $V^n_\lambda(t) = \max_{a \in A} Q^n_\lambda(t, a)$. 
The following theorem gives a sufficient condition for asymptotic convergence of the $Q$-learning algorithm.

\begin{theorem}[\citet{WatkinsDayan92}]
If every state-action pair in the environmental decision process is encountered infinitely often and the learning rates $0 \leq \alpha_n < 1$ satisfy the Robbins-Monroe conditions $\sum_{n=1}^{\infty} \alpha_n  = \infty$ and $\sum_{n = 1} ^{\infty} \alpha_n^2 < \infty$, then $Q^{n+1}_\lambda(s, a) {\to} Q_\lambda$ almost surely as $n {\to} \infty$.  
\end{theorem}

\subsection{Depreciating Assets}
We define variations on the discounted and average payoffs based on the idea that the value of an asset decays geometrically in proportion with the amount of time elapsed since it was obtained as a reward.
That is, we consider the situation in which a payoff is determined as a function of the sequence $\seq{R(s_n, a_n)}^\infty_{n = 1}$, but rather of the sequence 
\[
\seq{\sum^n_{k=1} R(s_k, a_k) \gamma^{n-k}}^\infty_{n = 1}
\] 
of exponential recency-weighted averages of the agent's assets, where $\gamma \in (0,1)$ is a discount factor.

%% file: past-disc.tex
In this section, we study discounted optimization, for $\lambda \in (0,1)$, under depreciating asset dynamics.
The payoff in this setting is captured by the expression
\begin{equation}
	\sum^\infty_{n=1} \lambda^{n-1} \sum^n_{k=1} R(s_k, a_k) \gamma^{n-k}, \tag{Discounted Depreciating Payoff}
\end{equation}
which has a corresponding value function
\begin{equation*}
	V_\lambda^\gamma(s) = \sup_{\pi \in \Pi^M} \bb{E}^\pi_s\bk{\sum^\infty_{n=1} \lambda^{n-1} \sum^n_{k=1} R(s_k, a_k) \gamma^{n-k}}.
\end{equation*}

Let us now return to the used car dealership example.

\begin{example}[Used Car Dealership Cont.]
Recognizing that cars depreciate continually after their first purchase, the employee realizes that their model should incorporate a notion of asset depreciation.
After a bit of market research, the employee selects another discount factor $\gamma \in (0,1)$ to capture the rate at which automobiles typically lose value over a given time step.
Using both discount factors $\lambda$ and $\gamma$, the employee can model the scenario as a discounted depreciating optimization problem.

For the sake of simplicity, suppose that there are only two locations $s_1$ and $s_2$ from which to choose the next target market, and that the only point where the employee has more than one possible action is at the dealership $s_d$ (from where they can chose action $a_1$ to go to $s_1$ or $a_2$ to go to $s_2$).
Realizing that it is unreasonable to plan without expecting unforeseen delays, the employee also introduces two parameters $\rho_1$ and $\rho_2$, which are success rates for buying a desired vehicle in $s_1$ and $s_2$ respectively.
Given that the agent is in location $s_i$, the rate $\rho_i$ is interpreted as the probability that they find a seller and purchase a vehicle before the end of the day and thus $1 - \rho_i$ is the probability that they fail to do so.
This situation is represented graphically as a finite MDP in \autoref{fig:car_mdp}, where actions are displayed in red, transition probabilities in blue, and immediate rewards (i.e. car values when they are stocked) in green.
If an action is omitted from an edge label, then there is only one action $a$ available.
If a transition probability is omitted, then the transition is deterministic, i.e. occurs with probability 1.
If a reward value is omitted, then the reward obtained is 0.

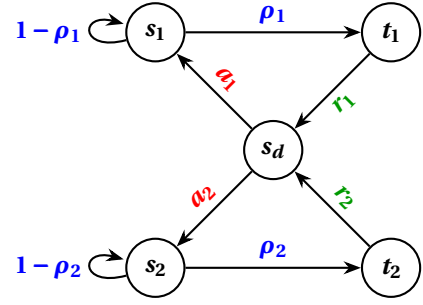
\begin{figure}
	\centering
	\begin{tikzpicture}[> = Stealth, thick]
		\node[state] (d) {$\bm{s_d}$};
		\node[state] (s1)[above left = 1cm and 1cm of d] {$\bm{s_1}$};
		\node[state] (t1)[above right = 1cm and 1cm of d] {$\bm{t_1}$};
		\node[state] (s2)[below left = 1cm and 1cm of d] {$\bm{s_2}$};
		\node[state] (t2)[below right = 1cm and 1cm of d] {$\bm{t_2}$};

		\path[->] (d) edge node[above,sloped,color=red] {$\bm{a_1}$} (s1);
		\path[->] (d) edge node[above,sloped,color=red] {$\bm{a_2}$} (s2);
		\path[->] (s1) edge[loop left] node[color=blue] {$\bm{1 - \rho_1}$} (s1);
		\path[->] (s1) edge node[above,color=blue] {$\bm{\rho_1}$} (t1);
		\path[->] (s2) edge[loop left] node[color=blue] {$\bm{1 - \rho_2}$} (s2);
		\path[->] (s2) edge node[above,color=blue] {$\bm{\rho_2}$} (t2);
		\path[->] (t1) edge node[sloped,below,color=darkgreen] {$\bm{r_1}$} (d);
		\path[->] (t2) edge node[sloped,above,color=darkgreen] {$\bm{r_2}$} (d);
	\end{tikzpicture}
	\caption{An MDP for the discounted depreciating optimization problem of the car dealership.}
	\label{fig:car_mdp}
\end{figure}

In traditional discounted optimization, the discount factor $\lambda$ imposes a certain type of trade-off.
Suppose, for instance, that $\rho_1$ is large while $r_1$ is small and that $\rho_2$ is small while $r_2$ is large.
Then a small discount factor indicates that it may payoff more to take action $a_1$ since it is likely that taking $a_2$ will result in significant delays and thus diminish the value of the eventual reward $r_2$.
On the other hand, if the discount factor is close to 1, then it may be worth it for the agent to accept the high probability of delay since the eventual discounted value will be closer to $r_2$.

Adding in the depreciation dynamics with discount factor $\gamma$, the trade-off remains, but to what extent depreciation alters the dynamics of a given environment and policy is unclear.
Intuition may suggest that introducing depreciation to discounted optimization should only make the risk-reward trade-off sharper, and one might further conjecture that when $\gamma$ is close to 0, the higher decay rate of cumulative asset value should drive an agent towards riskier behavior.
On the other hand, it is plausible that a depreciation factor close to one might embolden the agent towards similar risky actions because the opportunity cost of such behavior diminishes as assets are accumulated in greater quantities.
As we proceed with our analysis of the discounted depreciating payoff we attempt to shed light on questions like this and get to the core of what depreciation entails in this context.   
\end{example}

Our first main result establishes a Bellman-type equational characterization the discounted depreciating value.

\begin{theorem}[Optimality Equation]
	\label{thm:optimalityEQ}
	The discounted depreciating value is the unique solution of the equation
	\begin{equation}
		V_\lambda^\gamma(s) = \max_{a \in A} \frac{R(s, a)}{1 - \lambda\gamma} + \lambda \bb{E}_T\bk{V_\lambda^\gamma(t) \:\middle\vert\: s, a}.
	\end{equation}
\end{theorem}
\begin{proof}
	By splitting the term $\lambda^{n-1}$ occurring in the definition of the discounted depreciating payoff into the product $\lambda^{n-k}\lambda^{k-1}$ and distributing these factors into the inner summation, we obtain the expression
	\begin{equation}
		\label{eq:to_be_factored}
		\sum^\infty_{n=1} \sum^n_{k=1} \lambda^{k-1} R(s_k, a_k) \lambda^{n-k} \gamma^{n-k}.
	\end{equation}
	The next step of the proof relies on the following classical result of real analysis (c.f. Theorem 3.50 of \citet{Rudin76}).
	\begin{mertens}
		Let $\sum^\infty_{n=1} x_n = X$ and $\sum^\infty_{n=1} y_n = Y$ be two convergent series of real numbers.
		If at least one of the given series converges absolutely, then their Cauchy product converges to the product of their limits:
		\begin{equation*}
			\paren{\sum^\infty_{n=1} x_n} \paren{\sum^\infty_{n=1} y_n} = \sum^\infty_{n=1} \sum^n_{k=1} x_k y_{n-k} = XY.
		\end{equation*}
	\end{mertens}

	The series \eqref{eq:to_be_factored} may be factored into the Cauchy product 
	\begin{equation}
		\label{eq:cauchy_factored}
		\paren{\sum^\infty_{n=1} (\lambda\gamma)^{n-1}} \paren{\sum^\infty_{n=1} \lambda^{n-1} R(s_n, a_n)},
	\end{equation}
	and since both terms in this Cauchy product converge absolutely, Mertens' theorem applies.
	Thus, noticing that the left-hand series is geometric, the expression \eqref{eq:cauchy_factored} is equivalent to
	\begin{equation*}
		\frac{1}{1-\lambda\gamma} \sum^\infty_{n=1} \lambda^{n-1} R(s_n, a_n).
	\end{equation*}
	Consequently, the discounted depreciating value may be written as
	\begin{equation}
		\label{eq:FPtoF1}
		\begin{aligned}
			V_\lambda^\gamma(s) &= \sup_{\pi \in \Pi^M} \bb{E}^\pi_s\bk{\frac{1}{1-\lambda\gamma} \sum^\infty_{n=1} \lambda^{n-1} R(s_n, a_n)} \\
			&= \frac{1}{1-\lambda\gamma} \sup_{\pi \in \Pi^M} \bb{E}^\pi_s\bk{\sum^\infty_{n=1} \lambda^{n-1} R(s_n, a_n)} \\
			&= \frac{V_\lambda(s)}{1-\lambda\gamma}.
		\end{aligned}
	\end{equation}
	The equational characterization of the discounted value $V_\lambda$ now facilitates the derivation of the desired equational characterization of the discounted depreciating value $V_\lambda^\gamma$ as
	\begin{equation}
		\label{eq:FPtoF2}
		\begin{aligned}
			V_\lambda^\gamma(s) &= \frac{1}{1-\lambda\gamma} \paren{\max_{a \in A} R(s,a) + \lambda \bb{E}_T\bk{V_\lambda(t) \Mid s, a}} \\
			&= \max_{a \in A} \frac{R(s,a)}{1 - \lambda\gamma} + \lambda \bb{E}_T\bk{V_\lambda^\gamma(t) \Mid s,a}.
		\end{aligned}
	\end{equation}
\end{proof}

\noindent
An immediate consequence of \autoref{thm:optimalityEQ} is a characterization of the strategic complexity of discounted depreciating payoffs.

\begin{corollary}[Strategic Complexity]
	For any discounted depreciating payoff over any finite MDP, there exists an optimal policy that is stationary and deterministic. 
\end{corollary}

\autoref{thm:optimalityEQ} enables a number of extensively studied algorithmic techniques to be adapted for use under the discounted depreciating payoff.
In particular, the equational characterization of the discounted depreciating value implies that it is the unique fixed point of a contraction mapping \citep{banach1922operations}, which in turn facilitates the formulation of suitable variants of planning algorithms based on foundational methods such as value iteration and linear programming. 
This allows us to bound the computational complexity of determining discounted depreciating values in terms of the size of the environmental MDP and the given discount factors.

\begin{theorem}[Computational Complexity]
	The discounted depreciating value and a corresponding optimal policy are computable in polynomial time.
\end{theorem}
\begin{proof}
	Let $\delta_{i,j} = \begin{cases}
		1 &\tn{if } i=j \\
		0 &\tn{otherwise}
	\end{cases}$ be the Kronecker delta.
	Suppose that, for each state $s$ in the environment $M$, we have an associated real number $0 < x_s$, chosen arbitrarily.
	The unique solution to the following linear program is the vector of values from each state of $M$.
	\begin{equation}
		\label{eq:valueLP}
		\begin{aligned}
			&\tn{minimize } \sum_{s \in S} x_s v_s \quad\tn{subject to} \\
			&\frac{R(s,a)}{1-\lambda\gamma} \leq \sum_{t \in S} v_t \paren{\delta_{s,t} - \frac{\lambda T(t \mid s,a)}{1-\lambda\gamma}} &\forall (s,a) \in S \times A
		\end{aligned}
	\end{equation}
	From a solution $v^*$ to \eqref{eq:valueLP}, an optimal policy can be obtained as
	\begin{equation*}
		\pi(s) = \argmax_{a \in A} \frac{R(s,a)}{1-\lambda\gamma} + \lambda \bb{E}_T\bk{v^*_t \mid s,a}.
	\end{equation*}
	Alternatively, an optimal policy may be derived from the solution to the dual linear program given as follows.
	\begin{equation}
		\label{eq:policyLP}
		\begin{aligned}
			&\tn{maximize} \sum_{(s,a) \in S \times A} \frac{R(s,a)}{1 - \lambda\gamma} y_{s,a} \quad\tn{subject to} \\
			&x_s = \sum_{(t,a) \in S \times A} \paren{\delta_{s,t} - \frac{\lambda T(t \mid s,a)}{1-\lambda\gamma}} &\forall s \in S \\
			&0 \leq y_{s,a} &\forall (s,a) \in S \times A
		\end{aligned}
	\end{equation}
	In particular, if $y^*$ is a solution to \eqref{eq:policyLP}, then any policy $\pi$ for which the inequality $0 < y^*_{s,\pi(s)}$ holds at every state is optimal.
	The correctness of these linear programs follows from the proof of \autoref{thm:optimalityEQ}.
 Since linear programs can be solved polynomial time, the theorem follows.
\end{proof}

\autoref{thm:optimalityEQ} allows the formulation of an associated Q-value
\begin{equation*}
	Q_\lambda^\gamma(s, a) = \frac{R(s, a)}{1 - \lambda\gamma} + \lambda \bb{E}_T\bk{V_\lambda^\gamma(t) \:\middle\vert\: s, a},
\end{equation*}
which may be used to construct a Q-learning iteration scheme for discounted depreciating payoffs as
\begin{equation}
	\label{eq:Past-Q-update}
	\hspace{-3pt} Q^{\gamma, n+1}_\lambda(s, a) {\gets} Q^{\gamma, n}_\lambda(s, a) {+} \alpha_n \paren{\frac{R(s, a)}{1 {-} \lambda\gamma} {+} \lambda V^{\gamma, n}_\lambda(t) {-} Q^{\gamma,n}_\lambda(s, a)}.
\end{equation}

\begin{theorem}
	If each state-action pair of the environment is encountered infinitely often and the learning rates satisfy the Robbins-Monroe convergence criteria 
	\begin{equation*}
		\sum^\infty_{n=0} \alpha_n = \infty \quad\tn{ and }\quad \sum^\infty_{n=0} \alpha_n^2 < \infty,
	\end{equation*}
	then iterating \eqref{eq:Past-Q-update} converges almost surely to the discounted depreciating Q-value as $n \to \infty$:
	\begin{equation*}
		\lim_{n \to \infty} Q^{\gamma,n}_\lambda = Q_\lambda^\gamma.
	\end{equation*}
\end{theorem}
\begin{proof}
	Equations \eqref{eq:FPtoF1} and \eqref{eq:FPtoF2} show that the optimality equation for the discounted depreciating value reduces to the optimality equation for the discounted value, modulo a multiplicative factor dependent on $\lambda$ and $\gamma$.
	It therefore follows that discounted depreciating Q-learning, via iteration of \eqref{eq:Past-Q-update}, converges in the limit to the optimal $Q^\gamma_\lambda$ under the same conditions that standard discounted Q-learning, via iteration of \eqref{eq:Q-update}, converges in the limit to the optimal $Q_\lambda$.
	Hence, we conclude that discounted depreciating Q-learning asymptotically converges given that each state-action pair is encountered infinitely often and that the convergence conditions in the theorem statement are satisfied by the learning rates.
\end{proof}

\subsection{Discussion}
Besides the technical implications of \autoref{thm:optimalityEQ}, its proof provides some insight about the interplay between discounting and depreciation.
A foundational result \citep{BewleyKohlberg76} in the theory of infinite-horizon optimization establishes that over a common MDP the discounted value asymptotically approaches the average value, up to a multiplicative factor of $(1-\lambda)$, as $\lambda$ approaches 1 from below:
\begin{equation*}
	\lim_{\lambda \to 1} (1 - \lambda) V_\lambda = V.
\end{equation*}
Following this approach, we consider the asymptotic behavior of the discounted depreciating value when taking similar limits of the discount factors.
Using the identity $V_\lambda^\gamma = \frac{V_\lambda}{1 - \lambda\gamma}$ from equation \eqref{eq:FPtoF1} as the starting point for taking these limits yields the equations
\begin{gather}
	\lim_{\lambda \to 1} (1 - \lambda) V_\lambda^\gamma = \frac{V}{1 - \gamma}, \label{eq:FPtoM} \\
	\lim_{\gamma \to 1} V_\lambda^\gamma = \frac{V_\lambda}{1 - \lambda}, \label{eq:g_to_0} \\
	\lim_{\gamma \to 0} V_\lambda^\gamma = V_\lambda. \label{eq:g_to_1}
\end{gather}

\begin{figure}
	\centering
	\includegraphics[width=\columnwidth]{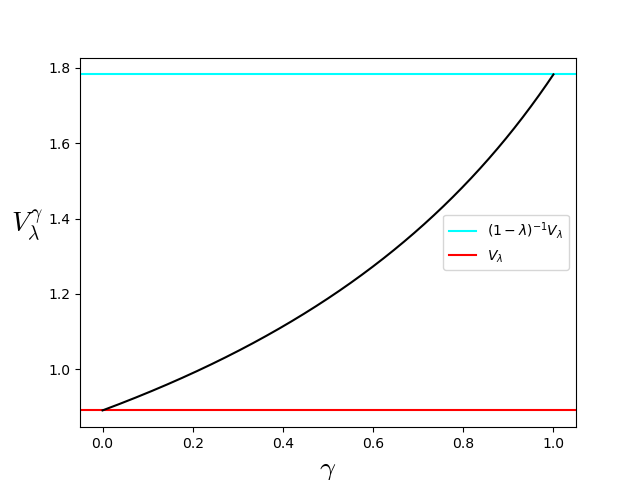}
	\caption{
		A graph of the discounted depreciating value of the car dealership example as $\gamma$ varies over the interval $(0,1)$ with fixed $\lambda = \frac{1}{2}$.
		The parameter values for this plot are $\rho_1 = \frac{1}{2}$, $\rho_2 = \frac{1}{4}$, $r_1 = 5$, $r_2 = 7$.
		}
		\label{fig:fixed_l_varied_g}
\end{figure}

The relationships described by equations \eqref{eq:g_to_1} and \eqref{eq:g_to_0}, illustrated by \autoref{fig:fixed_l_varied_g}, are justified conceptually by a simple interpretation that is helpful for building intuition around the behavior of the discounted depreciating payoff.
One can think of the standard discounted payoff as a special case of the discounted depreciating payoff where $\gamma = 0$.
That is, the optimizing agent working towards maximizing a discounted payoff does not consider the value of their assets whatsoever at any point in time; the only quantities of concern from their perspective are the incoming stream of rewards.
Interpreting $\gamma$ as a measure of the agent's memory of past outcomes, it follows naturally that the discounted depreciating payoff reduces to the discounted payoff when the agent has no recollection whatsoever.
Connecting this notion back to depreciation, it can be argued that, from the agent's perspective, externally driven depreciation of assets is morally equivalent to an internally driven perception of depreciation based on an imperfect recollection of past events.

Conversely, an agent with a perfect memory operating under a discounted payoff would end up maximizing this payoff on the sequence of cumulative assets $\seq{\sum^n_{k=1}R(s_k, a_k)}^\infty_{n=1}$ rather than the sequence $\seq{R(s_n, a_n)}_{n=1}^\infty$ of immediate rewards.
Assuming positive immediate rewards, this results in a greater value than would be obtained on the reward sequence itself, as evidenced by the plot in \autoref{fig:fixed_l_varied_g}.
As a consequence of the contraction property resulting from the standard discounting, the overall sum converges in spite of the fact that the cumulative asset stream may not be bounded.

\section {Average Depreciating Payoff}
\label{sec:average}
Let us now consider the asymptotic average evaluation criterion, given that assets depreciate.
The payoff of an outcome in this context is defined as
\begin{equation}
	\liminf_{n \to \infty} \sum^n_{k=1} \sum^k_{i=1} \frac{R(s_i, a_i) \gamma^{k-i}}{n}, \tag{Average Depreciating Payoff}
\end{equation}
and the associated average depreciating value function is
\begin{equation*}
	V^\gamma(s) = \sup_{\pi \in \Pi^M} \bb{E}^\pi_s\bk{\liminf_{n \to \infty} \sum^n_{k=1} \sum^k_{i=1} \frac{R(s_i, a_i) \gamma^{k-i}}{n}}.
\end{equation*}

\noindent
Our main result in this section asymptotically relates the average depreciating value and the discounted depreciating value.

\begin{theorem}[Tauberian Theorem]
	\label{thm:FP-MP}
 The limit of discounted depreciating value as $\lambda \to 1$ from below, scaled by $(1-\lambda)$, converges to the average depreciating value:
	\begin{equation*}
		\lim_{\lambda \to 1} (1 - \lambda) V_\lambda^\gamma = V^\gamma.
	\end{equation*}
\end{theorem}

The proof of \autoref{thm:FP-MP} uses the following pair of lemmas.

\begin{lemma}
	\label{lemma:helper1}
	For any finite path in the environmental MDP,
	\begin{equation}
		\label{eq:helper1}
		\sum^n_{k=1} \sum^k_{i=1} \frac{R(s_i, a_i) \gamma^{k-i}}{n} = \sum^n_{k=1} \frac{R(s_k, a_k) (1 - \gamma^{n+1-k})}{n(1 - \gamma)}.
	\end{equation}
\end{lemma}
\begin{proof}
	We proceed by induction on $n$.
	
	\item\paragraph*{Base case.}
	Suppose that $n=1$.
	Then both expressions occurring in \eqref{eq:helper1} evaluate to $R(s_1, a_1)$.

	\item\paragraph*{Inductive case.}
	Suppose that \eqref{eq:helper1} holds for $n-1$.
	By splitting the summation on the left-hand side of \eqref{eq:helper1}, we obtain the expression
	\begin{equation*}
		\sum^{n-1}_{k=1} \sum^k_{i=1} \frac{R(s_i, a_i) \gamma^{k-i}}{n} + \sum^n_{k=1} \frac{R(s_k, a_k) \gamma^{n-k}}{n}.
	\end{equation*}
	Factoring $\frac{n-1}{n}$ from the double summation in this expression yields
	\begin{equation*}
		\frac{n-1}{n} \sum^{n-1}_{k=1} \sum^k_{i=1} \frac{R(s_i, a_i) \gamma^{k-i}}{n-1} + \sum^n_{k=1} \frac{R(s_k, a_k) \gamma^{n-k}}{n}.
	\end{equation*}
	Now, applying the inductive hypothesis, this may be rewritten as
	\begin{equation*}
		\frac{n-1}{n} \sum^{n-1}_{k=1} \frac{R(s_k, a_k) (1 - \gamma^{n-k})}{(n-1)(1 - \gamma)} + \sum^n_{k=1} \frac{R(s_k, a_k) \gamma^{n-k}}{n}.
	\end{equation*}
	Factoring out $\frac{1}{n(1-\gamma)}$ from the entire expression, we get
	\begin{equation*}
		\frac{\sum\limits^{n-1}_{k=1} R(s_k, a_k) (1 {-} \gamma^{n-k}) + (1{-}\gamma)\sum\limits^n_{k=1} R(s_k, a_k) \gamma^{n-k}}{n(1-\gamma)}.
	\end{equation*}
	Distributing through the numerator results in the expression
	\begin{equation*}
		\hspace{-2pt} \frac{\sum\limits^{n{-}1}_{k{=}1} R(s_k{,}a_k) {-} R(s_k{,}a_k) \gamma^{n{-}k} {+} \sum\limits^n_{k{=}1} R(s_k{,}a_k) \gamma^{n{-}k} {-} R(s_k{,}a_k) \gamma^{n{+}1{-}k}}{n(1-\gamma)}
	\end{equation*}
	and removing those terms that cancel additively yields
	\begin{equation*}
		\frac{\sum\limits^{n}_{k{=}1} R(s_k, a_k) - \sum\limits^n_{k{=}1} R(s_k,a_k) \gamma^{n{+}1{-}k}}{n(1-\gamma)}.
	\end{equation*}
	Finally, we obtain \eqref{eq:helper1} by factoring the numerator one last time:
	\begin{equation*}
		\sum^n_{k=1} \frac{R(s_k,a_k) (1 - \gamma^{n+1-k})}{n(1-\gamma)},
	\end{equation*}
	thereby proving that if \eqref{eq:helper1} holds for paths of length $n-1$, then it also holds for paths of length $n$.
\end{proof}

\begin{lemma}
	\label{lemma:helper2}
	For any infinite path in the environmental MDP,
	\begin{equation*}
		\lim_{n \to \infty} \sum^n_{k=1} \frac{R(s_k, a_k) \gamma^{n+1-k}}{n(1-\gamma)} = 0.
	\end{equation*}
\end{lemma}
\begin{proof}
	Factoring out the constant term in the denominator of the left-hand side of the claimed equation, we obtain the equivalent expression
	\begin{equation*}
		\frac{1}{1-\gamma} \lim_{n \to \infty}\frac{1}{n} \sum^n_{k=1} R(s_k, a_k) \gamma^{n-k}.
	\end{equation*}
	Since the environmental MDP is assumed to be finite, there are finitely many possible reward values and we can bound the summation in the above expression as
	\begin{equation*}
		\frac{r_\downarrow (1 - \gamma^{n-1})}{1-\gamma} \leq \sum^n_{k=1} R(s_k, a_k) \gamma^{n-k} \leq \frac{r_\uparrow (1 - \gamma^{n-1})}{1-\gamma}
	\end{equation*}
	where $r_\downarrow = \min_{(s,a) \in S \times A} R(s,a)$ and $r_\uparrow = \max_{(s,a) \in S \times A} R(s,a)$.
	Lastly, noticing that
	\begin{equation*}
		\lim_{n \to \infty} \frac{r_\downarrow (1 - \gamma^{n-1})}{n(1-\gamma)} = \lim_{n \to \infty} \frac{r_\uparrow (1 - \gamma^{n-1})}{n(1-\gamma)} = 0,
	\end{equation*}
	it follows that
	\begin{equation*}
		\lim_{n \to \infty}\frac{1}{n} \sum^n_{k=1} R(s_k, a_k) \gamma^{n-k} = 0.
	\end{equation*}
\end{proof}

Now we are in position to prove \autoref{thm:FP-MP}.

\begin{proof}[Proof of \autoref{thm:FP-MP}]
	In light of equation \eqref{eq:FPtoM}, it is sufficient to prove the identity $V^\gamma = \frac{V}{1 - \gamma}$.
	Applying \autoref{lemma:helper1}, the average depreciating payoff may be rewritten as
	\begin{equation*}
		\liminf_{n \to \infty} \sum^n_{k=1} \frac{R(s_k, a_k) (1 - \gamma^{n+1-k})}{n(1 - \gamma)}.
	\end{equation*}
	Distributing the product in the numerator and then breaking the summation into a difference of summations yields the expression
	\begin{equation*}
		\liminf_{n \to \infty} \paren{\sum^n_{k=1} \frac{R(s_k, a_k)}{n (1-\gamma)} - \sum^n_{k=1} \frac{R(s_k, a_k) \gamma^{n+1-k}}{n (1-\gamma)}}.
	\end{equation*}
	By \autoref{lemma:helper2}, the right-hand term in this difference tends to 0 as $n \to \infty$, and so the above expression is equivalent to
	\begin{equation*}
		\liminf_{n \to \infty} \sum^n_{k=1} \frac{R(s_k, a_k)}{n (1-\gamma)}.
	\end{equation*}
	Factoring the constant term in the denominator out, the remaining limit-term is exactly the definition of the average payoff, and thus we conclude, for any state $s$, that
	\begin{equation*}
		V^\gamma(s) = \frac{V(s)}{1 - \gamma}.
	\end{equation*}
\end{proof}

As a direct consequence of \autoref{thm:FP-MP}, there exists a Blackwell optimal policy that is optimal for $V_\lambda^\gamma$ when $\lambda$ is sufficiently close to 1, that is also optimal for $V^\gamma$.

\begin{corollary}
	There exists a discount factor $\lambda_0 \in (0,1)$ and a policy $\pi$ such that, for all $\lambda \in [\lambda_0, 1)$ and every state $s$, it holds that
	\begin{align*}
		V_\lambda^\gamma(s) &= \bb{E}^\pi_s\bk{\sum^\infty_{n=1} \lambda^{n-1} \sum^n_{k=1} R(s_k, a_k) \gamma^{n-k}}, \\
		V^\gamma(s) &= \bb{E}^\pi_s\bk{\liminf_{n \to \infty} \frac{1}{n} \sum^n_{k=1}\sum^k_{i=1} R(s_i, a_i) \gamma^{k-i}}.
	\end{align*}
\end{corollary}

In turn, this implies the following result on the strategic complexity for the average depreciating payoff.

\begin{corollary}[Strategic Complexity]
	For any average depreciating payoff over any finite MDP, there exists an optimal policy that is stationary and deterministic. 
\end{corollary}